\documentclass[twoside,11pt]{article}
\usepackage[mathcal]{euscript}
\usepackage{bm,url}                     
\usepackage{amsfonts}
\usepackage{graphicx}               
\usepackage{natbib}                 
\usepackage{colortbl}               
\usepackage{booktabs}
\usepackage{amssymb}
\usepackage{mwe} 
\usepackage{caption}
\usepackage{amsmath}
\usepackage{xcolor}
\usepackage{amsthm}
\usepackage{todonotes}
\usepackage{bbold}
\newtheorem{definition}{Definition}
\newtheorem{remark}{Remark}
\newtheorem{theorem}{Theorem}

\newtheorem{lemma}{Lemma}
\newtheorem{proposition}{Proposition}

\usepackage[english]{babel}
\usepackage{graphicx}
\usepackage{color}
\usepackage{array}
\usepackage{mathrsfs}
\usepackage{setspace}
\usepackage{hyperref}

\newcommand{\eps}{\varepsilon}

\usepackage{amsfonts}

\newcommand{\X}{\mathcal{X}}
\newcommand{\Z}{\mathcal{Z}}
\newcommand{\T}{\mathcal{T}}
\newcommand{\Y}{\mathcal{Y}}

\newcommand{\unclass}{\mathcal{U}}

\DeclareMathOperator*{\argmax}{arg\,max}
\newcommand{\D}{\mathcal{D}}

\begin{document}

	\begin{center}
		\Large \bf On semi-supervised learning\\
	\end{center}
	\normalsize
	
	\
	
	\begin{center}
		Alejandro Cholaquidis$^a$, Ricardo Fraiman$^a$ and Mariela Sued $^b$\\
		$^a$ CABIDA  and Centro de Matem\'atica,\\ Facultad de Ciencias, Universidad de la Rep\'ublica, Uruguay\\
		$^b$ Instituto de C\'alculo, \\ Facultad de Ciencias Exactas y Naturales, Universidad de Buenos Aires\\
	\end{center}

	\begin{abstract}
	
	Semi-supervised learning deals with the problem of how, if possible, to take advantage of a huge amount of unclassified data, to perform a classification in situations when, typically, there is little labeled data. Even though this is not always possible (it depends on how useful, for inferring the labels, it would be to know the distribution of the unlabeled data), several algorithm have been proposed recently. 
	
	A new algorithm is proposed, that under almost necessary conditions, 
	attains asymptotically the performance of the best theoretical rule as the amount of unlabeled data tends to infinity. The set of necessary assumptions, although reasonable, show that semi-supervised classification only works for very well conditioned problems. The focus is on understanding when and why semi-supervised learning works when the size of the initial training sample remains fixed and the asymptotic is on the size of the unlabeled data. The performance of the algorithm is assessed in the well known ``Isolet'' real-data of phonemes, where a strong dependence on the choice of the initial training sample is shown.
\end{abstract}

\textit{Semi-supervised learning; Small training sample; Consistency.}

	\section{Introduction}\label{sec:intro}
	
Semi-supervised learning (SSL) dates back to the 60's,  starting with the pioneering works of \cite{scudder}, \cite{fralick} and \cite{agrawala},  among others. Later on, the problem was addressed by the highly influential works of \cite{cov,cov2}. The first one shows that  when the size $l$ of the unlabelled sample is equal to infinity, the classification error converges exponentially fast to the Bayes risk, if the size $n$ of the labelled sample converges to infinity. In the second one it is assumed that the density of the covariates is given by  a parametric model $p(x)=\pi p(x|y=\theta)+(1-\pi) p(x|y=1-\theta)$, where $p(x)$ is known except for the parameters $\theta\in \{0,1\}$ and  $\pi\in (0,1)$. 
 Under regularity conditions consistency is shown if the minimum between $n$ and $l$ converges to infinity.	

 Recently SSL has gained paramount importance due  to the huge amount of data coming from diverse sources, such as the internet, genomic research, text classification, and many others; see, for instance  \cite{zhu} or \cite{MIT} for a survey on SSL.
	This large amount of data is typically unlabelled; the main purpose of SSL is to jointly classify these data in the presence of a small ``training sample''. Namely, a lot of unlabelled data together with a small quantity of labelled data must  be combined to classify each unlabelled observation.  As   in \cite{arnold} ``A setting that is closely related to semi-supervised learning is transductive learning \cite{vap:98,joachims,joachims2}'', which is a special case of SSL, where the auxiliary unlabelled data-set coincides with the test sample. On the other hand, from  a transductive learning procedure any other data point can be classified with  any machine learning algorithm, by using as training sample the output of the transductive procedure.

	On the other hand, as discussed  in \cite{MIT}, the following question naturally arises: ``in comparison with a supervised algorithm that uses only labelled data, can one hope to have a more accurate prediction by taking into account the unlabelled points? [...]  In principle, the answer is yes''. Nevertheless, 
		 having a large set of data to classify is like knowing $p(x)$,  the distribution of the features vector; thus, the gain  in prediction accuracy   depends on the ability of $p(x)$ to provide information on $p(y|x)$. 
		 As it is pointed in \cite{chapzien}, ``the cluster assumption is key to successful semi-supervised learning'', which is expressed in terms of the so called \textit{valley condition}. Roughly speaking, this condition imposes $p(x)$ to have a deep valley between the classes.  In other words, clustering techniques have to perform reasonably  well in the presence of only  unlabelled data. 	Smoothness of the labels with respect to the features, or low density at the decision boundary, are examples of the kind of hypotheses required to get satisfactory results in the cluster analysis literature.

Another important issue in SSL is the amount of labelled data necessary to be able to classify the unlabelled data.  In the framework of generative models, 
		when  $p(x)$   is assumed to be an identifiable  mixture of parametric distributions, \cite{zhu} argued that  ``ideally we only need one labelled example per component''   to fully determine the mixture distribution. Indeed, under the regularity conditions presented in Section \ref{seccons},  one labelled example per component will also be  enough to prove the consistency of the algorithm that  we propose in this work.\\
	Recently, other approaches as self-training, co-training, transductive support vector machines, and graph-methods among others, have been reported.  Although there is a large body of literature on SSL, as it is pointed out by \cite{wass}, ``making precise how and when these assumptions actually improve inferences is surprisingly elusive, and most papers do not address this issue; some exceptions are \cite{rigollet}, \cite{singh}, \cite{lafferty}, \cite{nadler}, \cite{ben}, \cite{sinha}, \cite{belkin},  \cite{vap:98}, \cite{wsp:07} and \cite{niyogi}''. In \cite{haffari} the well known Yarowski   algorithm is analyzed, while  in \cite{wass} an interesting method called ``adaptive semi-supervised inference'' is introduced, and a minimax framework for the problem is provided.

		 \
		 
		Our proposal is focused on the case of a small and fixed  training sample size, but the amount of the unlabelled data goes to infinity (see Figure \ref{ilustracion}).
		We  provide a simple algorithm to classify the unlabelled data,   which has a resemblance to Yarowski's formulation. We prove that, under some natural and necessary conditions (some of them are in terms of well--known geometric constraints on the support of $p(x)$ coming from stochastic geometry), our method performs as good as the theoretical (unknown) best rule,  with probability one, asymptotically in $l$.  These conditions are discussed in Section \ref{assump} where we argue that  most of them seems to be necessary. 
	 
		The algorithm is of the ``self-training'' type; this means that at every step a point from the unlabelled set is labelled using the training sample  built up to that step, and incorporated into the training sample. In this way the training sample increases from one step to the next.    A simplified, computationally more efficient alternative algorithm is also provided in Section \ref{faster}.

 This paper is organized as follows: Section \ref{notation} introduces the basic notation and the set-up necessary to read the rest of the article. Section \ref{thrule} proves that the Bayes rule is the best one to classify the unlabelled sample. In Section \ref{alg} we introduce the algorithm and prove that all the unlabelled data are classified. Section \ref{seccons} proves that, as the number of unlabelled data grows to infinity, the algorithm performs as good as Bayes rule. In Section \ref{faster} we introduce a simplified and faster algorithm. Section \ref{sim} analyses two examples using simulated data, and a third one based on a real data set. Lastly, Section \ref{assump}  discusses the hypotheses. The proofs are included in Appendixes A and B.

\begin{figure}[htbp]
		\begin{center}
			\includegraphics[scale=.2]{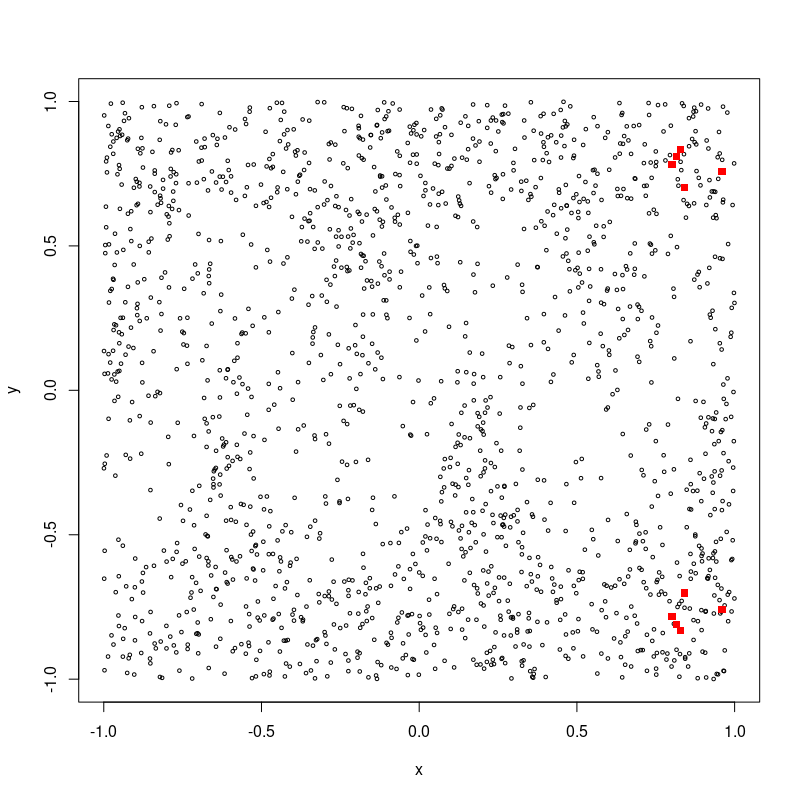} 
			\caption{In  black: the $X_j$ without labels, in red a small training sample (5 data from each subpopulation).}
			 \label{ilustracion}
		\end{center}
	\end{figure}

\section{Notation and set-up}\label{notation}

Along this work we use $\mathcal{I,A,B}$ to denote probability events,  namely, subsets of a (rich enough) probability space $(\Omega,\Sigma,\mathbb{P})$. Instead $I,A,B$ are used to denote subsets in the Euclidean space $\mathbb{R}^d$. In most of the cases, the probability events are defined through conditions on the random variables that concern $\mathbb{R}^d$. We use the same letter in different styles with the hope to facilitate the reading of the work.

We consider $\mathbb{R}^d$ endowed with the Euclidean norm $\|\cdot\|$. The open ball of radius $r\geq 0$ centered at $x$ is denoted by $B(x,r)$. With a slight abuse of notation, if $A\subset \mathbb{R}^d$, then we write $B(A,r)=\cup_{s\in A} B(s,r)$. The $d$-dimensional Lebesgue measure is denoted by $\mu_L$, while $\omega_d=\mu_L(B(0,1))$. For $\delta> 0$ and $A\subset \mathbb{R}^d$, the $\delta$-interior of $A$ is defined as $A\ominus B(0,\delta)=\{x:B(x,\delta)\subset A\}$. The distance from a point $x$ to a set $A$ is denoted by $d(x,A)$, i.e. $d(x,A)=\inf\{\|x-a\|:a\in A\}$. If $A\subset\mathbb{R}^d$, then $\partial A$ denotes its boundary, $int(A)$ its interior, $A^c$ its complement, and $\overline{A}$ its closure.
Let $\mathcal{D}^n=(\X^n,\Y^n)=\hspace{-0.1cm}\{(X^1,Y^1),\dots,(X^n,Y^n)\}$ be  a given  realization of a sample with the same distribution as $(X,Y)\in S\times \{0,1\}$,  where 
$S\subset \mathbb{R}^d$. We assume that they are identically distributed but not necessarily independent.
Let  $\eta(x)$ denote the conditional mean of $Y$ given $X=x$; namely, $\eta(x)= \mathbb{E}(Y|X=x)$. Consider $\mathcal{D}_l=(\X_l,\Y_l)=\{(X_1,Y_1),\dots,(X_l,Y_l)\}$ an iid sample with the same distribution as $(X,Y)$, where $n\ll l$. The sample $\X_l=(X_1,\dots,X_l)$ is known while  the labels $\Y_l=(Y_1, \ldots, Y_l)$ are unobserved. 

\section{Theoretical best rule}\label{thrule}

It is well known that the optimal rule for classifying a single new datum $X$ is given by the Bayes rule, $g^*(X)=\mathbb{I}_{\{\eta(X)\geq 1/2\}}$.  In the present paper, we move from the classification problem of a single datum $X$ to a framework where each coordinate of  $\X_l=(X_1,\dots,X_l)$
must be classified. The label associated with each coordinate $X_i$  may be constructed
on the basis of the entire vector and, therefore, a classification rule  $\mathbf{g}_l=(g_1,\dots,g_l)$ comprises $l$ functions $g_i:S^l\rightarrow \{0,1\}$, where $g_i(\X_l)$ indicates the 
label assigned to $X_i$ based on the entire set of observations $\X_l$. The performance of a rule $\mathbf{g}_l=(g_1,\dots,g_l)$ is given by its mean  classification error, namely $L({\mathbf{g}_l}):=\mathbb{E}\Big(\frac{1}{l}
\sum_{i=1}^l
\mathbb{I}_{g_i(\X_l) \neq Y_i}\Big).$
Observe that the random variable $\#\{i:g_i(\X_l)\neq Y_i, (X_i,Y_i)\in \D_l\}$ is not necessarily $Binomial(l,p)$ for some $p\geq 0$.

The next result establishes that the optimal classification rule classifies each element ignoring the presence of the rest of the observations, 
by means of invoking the  Bayes rule.

   \begin{proposition} \label{prop0} The performance of a rule $\mathbf{g}_l$ is bounded from below by $L^*= \mathbb{P}(g^*(X)\neq Y)$, and the lower bound  is attained with the rule $\mathbf{g}^*_l=(g_1^*,\dots,g_l^*)$,
   	where $g_i^\ast(\X_l)=g^\ast(X_i)$ for all $i=1,\dots,l$. 
   \end{proposition}
   
 In practice, since the distribution of $(X,Y)$ is unknown, an estimator of $\mathbf{g}^*_l$  may be defined by a sequence 
   $\mathbf{g}_{n,l}=(g_{n,1},\dots,g_{n,l})$, where  $g_{n,i}:S^l\times (S\times \{0,1\})^n\to \{0,1\}$   indicates the label to be assigned to the element $X_i\in \mathcal X_l$. 
   In what follows we try to find a sequence $\mathbf{g}_{n,l}(\mathcal X_l, \mathcal D^n)$, such that
\begin{equation} \label{consistency}
\lim_{l\rightarrow \infty} \mathbb{E}_{\mathcal{D}_l}\Big(\frac{1}{l}\sum_{i=1}^l \mathbb{I}_{g_{{n,i}}(\X_l,\mathcal{D}^n)\neq Y_i}\Big)- L(\mathbf{g}^*_l)=0, \text{ for a fixed realization }\mathcal{D}^n,
\end{equation}
where $\mathbb{E}_{\mathcal{D}_l}$ denotes the expectation wrt $\mathcal{D}_l$.

\begin{remark}
	It can be surprising that the limit in display \eqref{consistency} does not depend explicitly on the size $n$ of the initial training sample $\mathcal D^n$. Our purpose is to analyze under which conditions such a strong statement can be derived. As it is proved in Theorem \ref{teoconst},  the initial training sample must be well located in the sense of assumption H8 given below.  Moreover,   strong but almost necessary assumptions discussed in Section \ref{assump} are required to  get the desired result.
\end{remark}

Next section presents an algorithm  that, under several conditions (given in section \ref{seccons}), satisfies a stronger property. More precisely, we will show that 
\begin{eqnarray*}
\lim_{l\rightarrow \infty}	\frac{1}{l}\sum_{i=1}^l \mathbb{I}_{g_{{n,i}}(\X_l,\mathcal{D}^n)\neq Y_i} = \mathbb{P}\{g^*(X)\neq Y\} \quad a.s.  \end{eqnarray*}
where $g_{n,i}=g_{n,l,r(i)}$ and $r(i)$ is the step of the algorithm at which the point $X_i$ is classified.
 
\section{Algorithm}\label{alg}

We provide an  algorithm  which is asymptotically optimal in the sense of satisfying condition \eqref{consistency}.
For this purpose, we update the training sample sequentially incorporating   into the initial set $\D^n$ an observation  $X_{j_i}$ in  $\X_l$ with a predicted label $\tilde Y_{j_i} \in \{0,1\} $. 
At each step we choose the point whose score to predict its label  is as extreme as possible, as stated in display \eqref{ji}. 
Scores  are constructed according to the  majority rule in a neighborhood of the corresponding observations to be classified; i.e., we estimate $\eta(x)$  with a Nadaraya-Watson estimator  using  a uniform kernel,  based on  both $\mathcal D^n$ and those points already classified by the algorithm up to the present step. In this way  we choose the  ``\textit{best classifiable point}'' from those that remain unclassified, as indicated in the following recipe:

\begin{itemize}\item[Initialization:]
	Let $\Z_0=\X^n$, $\unclass_0=\X_l$,  $\T_0=\D^n$.

	\item[STEP $j$:] For $j$ in $\{1,\ldots,l\}$,
	choose the \textit{best classifiable point} in $\unclass _{j-1}$, from those that are at a distance smaller than $h_l$ from the  points already classified, as follows: 
	let $\unclass_{j-1}(h_l)=\{ X\in \unclass_{j-1}: d(\Z_{j-1}, X)< h_l\}$;  
	for $X_i \in \unclass_{j-1}(h_l)$, consider  
	\begin{multline}
		\label{eta_i} 
		\hat{\eta}_{j-1}(X_i)=\\ \frac{\sum\limits_{\{r:(X^r,Y^r)\in \D^n\}}\! Y_r\mathbb{I}_{B(X_{i},h_l)}(X^r)+\sum\limits_{\{r:(X_r,\tilde{Y}_r)\in \T_{j-1}\setminus \D^n\}} \!\tilde{Y}_r\mathbb{I}_{B(X_{i},h_l)}(X_r)}{\sum\limits_{\{r:(X^r,Y^r)\in \D^n\}}\mathbb{I}_{B(X_{i},h_l)}(X^r)+\sum\limits_{\{r:(X_r,\tilde{Y}_r)\in \T_{j-1}\setminus \D^n\}}\mathbb{I}_{B(X_{i},h_l)}(X_r)} , 
	\end{multline}
	\begin{equation} 
			\label{ji}\text{ 	and define  } X_{i_j}= \argmax_{i:X_i\in \unclass_{j-1}(h_l)} \max\Big\{\hat{\eta}_{j-1}(X_i),1-
		\hat{\eta}_{j-1}(X_i)\Big\}.
	\end{equation}
	 If there is more than one $i_j$ satisfying \eqref{ji},  choose one that maximizes 
	 \begin{equation}\label{card}
	 \#\{\X_l\cap B(X_{i_j},h_l)\}.
	 \end{equation}
	  Then label $X_{i_j}$ with $\tilde{Y}_{i_j}$ defined by $\tilde{Y}_{i_j}=
	g_{n,l,j-1}(X_{i_j})$, where $g_{n,l,j-1}$ is the classification rule associated with $\hat{\eta}_{j-1}$ defined in  \eqref{eta_i}. Namely, $\tilde{Y}_{i_j}= \mathbb{I}_
	{\{\hat{\eta}_{j-1}(X_{i_j})\geq1/2\}}$.  Consider
	$$\Z_j=\Z_{j-1}\cup \{X_{i_j}\} ,\quad \unclass_j=\unclass_{j-1}\setminus \{X_{i_j}\} \quad\hbox{and} \quad  \T_{j}=\T_{j-1}\cup \{(X_{i_j},\tilde{Y}_{i_j})\} .$$

	\item[OUTPUT:] $\{(X_{i_1},\tilde{Y}_{i_1}),\dots,(X_{i_l},
	\tilde{Y}_{i_l})\}$.
	\end{itemize}

Alternatively, to reduce the computational time,  in Step $j$, instead of choosing only one point satisfying \eqref{ji} and maximizing \eqref{card}, it is possible to choose, among the points that satisfy \eqref{ji}, all those fulfilling \eqref{card}. More precisely, we define $\aleph_j$ as the set of all the points that satisfy \eqref{ji} and $\Gamma_{j}=\{X_{1_j},\dots,X_{m_j}\}\subset \aleph_{j}$ that maximize $\#\{\X_l\cap B(X_{r_j},h_l)\}$.
Then we label $X_{1_j},\dots,X_{m_j}$ with $\tilde{Y}_{1_j},\dots,\tilde{Y}_{m_j}$ defined by $\tilde{Y}_{r_j}=g_{n,l,j-1}(X_{r_j})$ for all $X_{r_j}\in \Gamma_j$, where $g_{n,l,j-1}$ is the classification rule associated with $\hat{\eta}_{j-1}$ defined in  \eqref{eta_i}. More precisely, $\tilde{Y}_{r_j}= \mathbb{I}_
{\{\hat{\eta}_{j-1}(X_{r_j})\geq1/2\}}$. Lastly $\Z_j=\Z_{j-1}\cup\Gamma_j ,\quad \unclass_j=\unclass_{j-1}\setminus \Gamma_j \quad\hbox{and} \quad  \T_{j}=\T_{j-1}\cup \{(X_{1_j},\tilde{Y}_{1_j}),\dots,(X_{m_j},\tilde{Y}_{m_j})\} .$

The results discussed in the remainder of this work 
hold for both versions of the algorithm. To simplify the notation, they are only presented for the first version, labelling one point at each step. However, the data analysis developed in Section \ref{sim}
is based on the second version of the algorithm.

We will now  prove that the algorithm 
classifies the whole set $\X_l$. For that purpose,  
define $I_0=\eta^{-1}\left\{[0,1/2)\right\},\ I_1=\eta^{-1}\left\{(1/2,1]\right\},$
and assume that $I_0$ and $I_1$ are connected and \textit{coverable}, as stated in condition H3 below. Observe that $I_1\cup I_0\cup\eta^{-1}(1/2)= S$, where  $S$ is assumed to be the support of the random vector  $X$. We decided to include H3 to facilitate the proof of Proposition \ref{alg2}. In Proposition \ref{mariela} we will provide sufficient conditions which guarantee the  validity of H3. Such conditions are  expressed in terms of geometric restrictions on $I_a$, $a=0,1$,  regularity assumptions on the density function $f$ of the distribution of $X$, and on the rate at which the  bandwidth  $h_l$ decreases to zero. These conditions  will also  be discussed  in Section \ref{assump}. Additionally, we  require to have  at least one point of the training sample in $I_a$, for $a=0,1$.  
To be more precise, consider the following assumptions: 
\begin{itemize}

		\item[H1.] $\mathbb {P}\{X\in \eta^{-1}(1/2)\}=0$.

	\item[H2.] For $a=0,1$, $i)$  $I_a$ is  connected, and $ii)$  $\mathbb {P}(X\in I_a)>0 $
	
\item[H3.]  The covering property: the probability event $\mathcal I_a$ fulfills $\mathbb{P}( \mathcal I_a)=1$, for $a=0,1$, where, 
		\begin{eqnarray*}
		\mathcal I_a=\bigcup_{l_0}\bigcap_{l\geq l_0}\mathcal I_{a,l} \quad \hbox{and}\quad  
			\mathcal I_{a,l}=\left\{\omega \in \Omega:I_a\subseteq \bigcup_{X\in \X_{l}\cap I_a } B(X,h_l/2)\right\},l\in \mathbb N.
		\end{eqnarray*}
	
		\item[H4.] There exists $X_a^\ast$ in $\D^n$ such that $X_a^\ast\in I_a$, for $a=0,1$.  
\end{itemize}

In the sequel, we will assume H1 and therefore $\mathbb {P}(X\in I_0\cup I_1 )=1$. We can now establish that, for $l$ large enough, the algorithm assigns labels to each point  in  $\X_l$.

\begin{proposition} \label{alg2}  Assume $H1$, $H2$ i), $H3$ and $H4$. Then, with probability one, for $l$ large enough, all the points in $\X_l$ are classified by the algorithm: $\mathbb{P}(\mathcal F)=1$, where 
 $\mathcal F=\cup_{L=1}^\infty\cap_{l=L}^\infty \mathcal F_l$ and, for  $l \in \mathbb N$,  $\mathcal F_l=\{\omega: \X_l(\omega) \;\hbox{is entirely  classified}\}$.
\end{proposition}

\begin{remark} If we estimate $\eta$ with a  $k$-nearest neighbor rule instead of the kernel procedure  proposed in this work, the result presented in  Proposition \ref{alg2} holds with no assumptions.

More generally, in the algorithm $\hat{\eta}_{j-1}(X_i)$ might be replaced by other local nonparametric regression estimator like neural networks. However, among other conditions, the diameter of the partition cells will play an important role. 
The analysis of consistency of the proposed algorithm based on $k$-nn or neural network rules is beyond the scope of this manuscript.
  
\end{remark}

\section{Consistency of the algorithm} \label{seccons}
 
To prove the consistency of the algorithm additional  conditions are required. They  involve  regularity properties of different sets and  the rate at which $h_l$ decreases. Define the following sets,  illustrated in Figure \ref{IAB}: 
\begin{align*}
 A^\delta_0=I_0\ominus B(0,\delta)\;,\quad  A^\delta_1=&I_1\ominus B(0,\delta)\;,\\
 B_0^h=I_0\cap B(I_1,h)\;,\quad  B_1^h=&I_1\cap B(I_0,h)\;.
 \end{align*}
Besides  H1-H4 introduced in Section \ref{alg},  we will also assume that both the $\delta$-interior  $A_0^\delta$ and $A_1^\delta$ of  $I_0$ and $I_1$, respectively,   are connected and \textit{coverable}, as stated in H5. This hypothesis (as we will see in Appendix B) is fulfilled if we assume that the set  $\overline{I_a^c}$, $a=0,1$,  has positive reach, (as introduced in \cite{federer:59}) and $lh_l^d/\log(l)\rightarrow \infty$. Assumption H7 holds if  $lh_l^{2d}/\log(l)\rightarrow \infty$, as it is proved in \cite{ab:89}.   Moreover, the density $f$ of the distribution of  $X$ needs to take  larger values  on the interiors $A_0^\delta\cup A_1^\delta$ than on the borders $B_0^h\cup B_1^h$, as indicated in H6. Finally, all the labels in the training set $\D^n$ must agree with those determined by the Bayes's rule, apart from being well located, as presented  in  H8. Namely, consider the following set of hypotheses, which  will be  discussed  in Section \ref{assump}:

\begin{itemize}
			\item[H5.] There exists $\delta_0>0$ such that, for $a=0,1$ and  for any $\delta<\delta_0$, $i)$ $A_a^\delta$ is connected, and $ii)$ the probability event $\mathcal A_a^\delta$ fulfills  $\mathbb{P}( \mathcal A_a^\delta)=1$,  where   
	\begin{eqnarray}
	\label{A_set}
	 \mathcal A_a^\delta=\bigcup_{l_0}\bigcap_{l\geq l_0}\mathcal A_{a,l}^\delta \hbox{ and }  
			\mathcal A^\delta_{a,l}=\left\{\omega\in \Omega: A_a^\delta\subseteq \bigcup_{X\in \X_{l}\cap A_a^\delta } B(X,h_l/2)\right\},
		\end{eqnarray}
for $l\in \mathbb N$.

	\item[H6.] The Valley Condition: The  probability function $P_X$ induced by $X$ has a density $f$ verifying that $\delta_1>0$ exists such that for all $\delta<\delta_1$ there is $\gamma=\gamma(\delta)>0$, such that when  $h<\delta$
\begin{equation} \label{h4}
f(a)-f(b)>\gamma\;,  \text{ for all } { a\in A_0^{\delta}\cup A_1^{\delta }}\text{ and all } {b \in  B_1^{h}\cup B^{h}_0}.
\end{equation}

		\item[H7.] The kernel density estimator $\hat f_l (u)=(\omega_dlh^d)^{-1}\sum_{i=1}^l \mathbb{I}_{B(u,h_l)}(X_i)$ converges to $f(u)$ uniformly  over its support $S$, almost surely: 
			\begin{equation}
			\label{conv_unif}
			\mathbb{P}\left( \bigcup_{l_0}\bigcap_{l\geq l_0} \sup_{u \in S}|\hat f_l(u)-f(u)|<\varepsilon \right)=1\;,  \forall \varepsilon>0. 
			\end{equation}

	\item[H8.] 	Good training set:  $Y^i=g^*(X^i)$ for all $(X^i,Y^i)\in \D^n$; moreover, there exists $X_a^\ast$ in $\D^n$ such that $X_a^n\in A_a^{\delta_2}$, for $a=0,1$, for some $\delta_2>0$. Observe that H8 implies H4.

\end{itemize}

Even if no condition is imposed on the bandwidth $h_l$, the algorithm implicitly assumes  that it converges to zero. Indeed, in Proposition \ref{mariela}, we ask for  rates of convergence to guarantee the validity of condition H3, H5 and H7, aside from
 some regularity conditions on $f$ and the sets $I_a$ for $a=0,1$.\\
 Following the notation in \cite{federer:59}, let $\text{Unp}(S)$ be the set of points $x\in \mathbb{R}^d$ with a unique projection on $S$, denoted by $\pi_S(x)$. That is, for $x\in \text{Unp}(S)$, $\pi_S(x)$ is the unique point that achieves the minimum of $\|x-y\|$ for $y\in S$. For $x\in S$, let {\it{reach}}$(S,x)=\sup\{r>0:B(x,r)\subset \text{Unp}(S)\big\}$. The reach of $S$ is defined by $\text{reach}(S)=\inf\big\{\text{reach}(S,x):x\in S\big\},$ and $S$ is said to be of positive reach if $\text{reach}(S)>0$.

\begin{proposition} \label{mariela} Assume that H2 i) and ii) hold and that  $f$ is compact supported, continuous, bounded from below by a positive constant. Assume also that $reach(\overline{I_a^c})>0$, for $a=0,1$. The bandwidth $h_l$ fulfills  $h_l\rightarrow 0$ and $lh_l^{2d}/\log(l)\rightarrow \infty$.
	 Then H3, H5 and H7 hold.
	\end{proposition}

The main result of this work is presented in Theorem \ref{teoconst}; it states 
 that  the  algorithm proposed in Section \ref{alg} is consistent, in the sense defined in (\ref{consistency}). 
To prove this result, we will invoke the following preliminary lemmas. 
 The first of them, Lemma \ref{lem1}, establishes that  the first point classified differently from the Bayes rule is in the boundary region $B_1^h\cup B_0^h$. Then, in Lemma \ref{lemaux}, we combine the  valley condition with the uniform consistency of the kernel estimator to show that, asymptotically, there are more points of $\X_l$ in 
$A_0^\delta\cup A_1^\delta$ than in $B_0^{h_l}\cup B_1^{h_l}$. Lemma  \ref{lem2} states that 
all the points far enough from the boundary region are labeled by the algorithm, with the same label that the one given by the Bayes rule. To be more precise, 
recall that,  $\mathcal F_l=\{\omega: \X_l(\omega) \quad\hbox{is entirely  classified}\}$ and define 
$$\mathcal B_l=\{\omega: \text{ there exists } X_{i_j}\in \mathcal{X}_l: \tilde{Y}_{i_j}\neq g^*(X_{i_j})\}\cap \mathcal F_l. $$
Look at the first time, $j_{bad}$, where the algorithm  assigns a label different from that prescribed by the Bayes rule, if such a step exists; otherwise, define $j_{bad}=\infty$. Namely,     
\begin{equation}
\label{elijo_mal}
j_{bad}=\inf\{j: \tilde Y_{i_j}\not=g^*(X_{i_j})\}\quad\hbox{on $\mathcal B_l$,}\quad \hbox{ and $j_{bad}=\infty$ on $\mathcal B_l^c$.}
\end{equation}
From now on, we will say that a point  $X_{i_j}\in \X_l$ is  \textit{badly classified} whenever  $\tilde Y_{i_j}\not=g^*(X_{i_j})$; otherwise the point will be called well classified.  The next result establishes that  $X_{i_{j_{bad}}}$ is in $B_0^{h_l}\cup B_1^{h_l}$.

\begin{lemma} \label{lem1} Assume that H1 and H8 hold. Then,  $\mathcal B_l \subset \{X_{i_{j_{bad}}}\;\in\; B_0^{h_l}\cup B_1^{h_l}\}.$
\end{lemma}

\begin{lemma} \label{lemaux} Assume H6 and H7. Then, 
$\mathbb{P}(\mathcal V^\delta)=1$, for any $\delta<\delta_1$ where 
	\begin{equation*}\label{eqlem}
	\mathcal V^\delta=\bigcup_{l_0}\bigcap_{l\geq l_0} \mathcal{V}_l^\delta 
\quad\hbox{and}	
	\end{equation*}
	$$	\mathcal{V}_l^\delta=\left\{\omega\in \Omega:\inf_{a\in A_0^\delta\cup A_1^\delta }\sum_{i=1}^l \mathbb{I}_{B(a,h_l)}(X_i)\;\geq \;\sup_{b\in  B_0^{h_l}\cup B_1^{h_l} }\sum_{i=1}^l \mathbb{I}_{B(b,h_l)}(X_i)\right\}.$$
	\end{lemma}

\begin{lemma} \label{lem2}  Assume H1--H8. Then,  for any $\delta<\min\{\delta_0,\delta_1,\delta_2\}$ 
			\begin{equation}
			\label{prop_10_new}
		\mathcal F_l \cap \mathcal A_{a,l}^\delta \cap \mathcal V^\delta_l \;\subset\; \left\{\X_l\cap A_a^\delta \cap (\Z_{j_{bad}-1})^c=\emptyset  \right\}\;, \quad a=0,1, 
			\end{equation}
and therefore, on $\mathcal F_l \cap \mathcal A_{0,l}^\delta \cap \mathcal A_{1,l}^\delta \cap \mathcal V^\delta_l$, we have that 
\begin{equation}   
\label{eq4}
\mathbb{I}_{\tilde Y_{i}= g^*(X_{i})}\geq \mathbb{I}_{A_0^\delta\cup A_1^\delta}(X_{i})\;,\quad i=1, \ldots, l.
\end{equation}
 
\end{lemma}

 	
\begin{theorem} \label{teoconst} 
 Assume that   $\mathcal{D}^n$ is a good training set, in the sense that fulfills  H8. 
 Then, under H1--H3, H5--H7,  the algorithm presented in Section \ref{alg} 
 satisfies
 	\begin{eqnarray*}
    		\lim_{l\rightarrow \infty}	\frac{1}{l}\sum_{i=1}^l \mathbb{I}_{g_{{n,l,r(i)}}(\X_l)\neq Y_i} = \mathbb{P}\{g^*(X)\neq Y\}\quad \text{a.s.}  \end{eqnarray*}
and therefore, it is consistent, as defined in \eqref{consistency}. 
   	\end{theorem}

\section{A faster algorithm} \label{faster}

The algorithm given in Section \ref{alg} classifies a few points of $\X_l$ at each step. This can be discouraging when $l$ is too large. In order to overcome this issue,  we will introduce a simple modification that gives rise to a faster procedure in terms of computational time (see Table \ref{tabgau}), at the expense of introducing a small increment in the classification error rate (this increment can be controlled but with computational cost). 

The idea is to pre-process the sample $\mathcal{X}_l$, and \textit{project it} on a grid $G_l$, as we describe in what follows. We can assume, without loss of generality, that $\mathcal{X}_l\cup \mathcal{X}^n\subset (a,b)^d$ with $a<b$.  For $N$ fixed, to be determined by the practitioner, consider $a_i=a+i(b-a)/N$ for $i=0,\ldots,N$. The $N$-grid $G_l$ on $(a,b)^d$ is  determined by the $N^d$ points of the form $\bold{a}=(a_{i_1},\dots,a_{i_d})$ with $i_j\in \{0,\dots,N-1\}$, for  $j=1,\dots,d$. Each point $\bold{a}$ in the grid determines a cell $C_\bold{a}=\prod_{j=1}^d (a_{i_j}, a_{i_j+1}].$

Given $\X_l$, let   ${T}_l$ be  the set of points $\bold{a}$ in the grid $G_l$ whose corresponding cell $C_{\bold{a}}$ intersects $\X_l$; now project (or collapse)  $\X_l $ on  $T_l$, in the sense that the algorithm will be applied to $T_l$ in lieu of $\X_l$. Then, all the points in $\X_l\cap C_{\bold{a}}$ will be classified with the label assigned to $\bold{a}$ by the algorithm.

\section{Examples with simulated and real data} \label{sim}

 In this section we report some numerical results, comparing  the performance of 
 the SSM algorithm presented in Section \ref{alg} and  its faster version, introduced in Section \ref{faster}, 
 with that of other supervised   algorithms. Specifically,   $k$-nearest neighbors  ($k$-nn)   and support vector machines (SVM)  are the supervised techniques used to assign  labels of each element in $\X_l$  on the basis of the training sample $\D^n$.

The classification error rate of each algorithm is computed in 
three scenarios. In the first two, we use artificially generated data, whereas in the last one we employ a real data set. 
The first example compares efficiency of the three algorithms   ($k$-nn, SVM and the SSL algorithm introduced in section \ref{alg}).   The second one shows the effect of the grid size with respect to 
classification error rate  and computational time. The third one is a well known real-data set  where we illustrate the crucial effect of the initial training sample $\mathcal D^n$.

\subsection{A first simulated example}

The joint distribution of $(X,Y)$ is generated as follows: consider first the curve $C$ in the square $[-1,1]^2$,  defined by $C=\{(x,(1/2)\sin(4x)):-1\leq x\leq 1\}$.   All the points in the square that are below $C$ will be labeled with $Y=0$ while   those that are above the curve $C$ will be labeled with $Y=1$. Now, to emulate the valley condition, those points close to $C$  will be chosen with less probability than those far away. To do so, 
	 let $S_1$ and  $S_2$ denote the set of  points in the square which are at $\|\cdot\|_\infty$-distance  larger  / smaller  than $0.2$ from $C$, respectively. 
	 Namely,  $S_1=\{B_{\|\cdot\|_\infty}(C,0.2)\}^c\cap [-1,1]^2$ and $S_2=B_{\|\cdot\|_\infty}(C,0.2)\cap [-1,1]^2$, where $\|\cdot\|_\infty$ is the supremum norm. Let $U_1$, $U_2$ and $B$ be independent random variables, with $U_1\sim \text{Uniform}(S_1)$, $U_2\sim \text{Uniform}(S_2)$ and  $B\sim Bernoulli(7/8)$.
  Consider the random variable $X=BU_1+(1-B)U_2$,  while   $(X,Y)=((X_1,X_2),1)$ if $X_2>(1/2)\sin(4X_1)$ and $(X,Y)=((X_1,X_2),0)$ if $X_2\leq (1/2)\sin(4X_1)$.

We first study the performance of our algorithm,  analyzing  the error rates among 50 replications of the described scheme, with $n=20$, $l=2400$ and $h_l=0.15$.  An histogram of the classification errors is presented on the right panel of Figure \ref{hist} and a  summary is reported   in Table \ref{errej1}. 
There are four out  of the fifty replications where the  classification errors are  much higher than  in the other cases.  
These extreme results can be attributed to the initial training sample $\mathcal D^n$  (see assumption H8). The initial training sample for the best and the worst case (in terms of classification error rate) are also shown on the left panels of  Figure \ref{hist}.

\begin{figure}[h]
	\begin{center}
			\includegraphics[scale=.18]{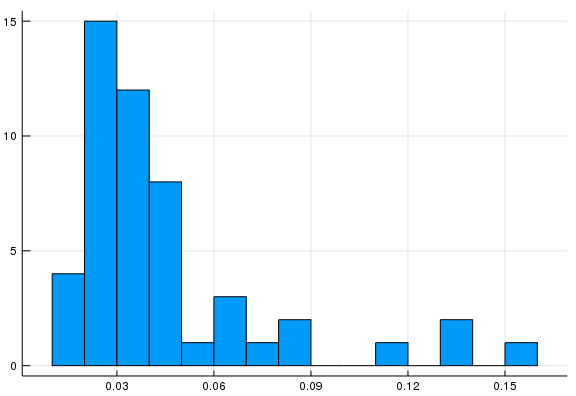}
						\includegraphics[scale=.18]{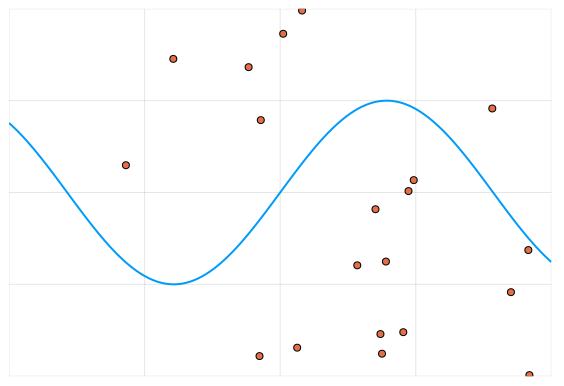}
									\includegraphics[scale=.18]{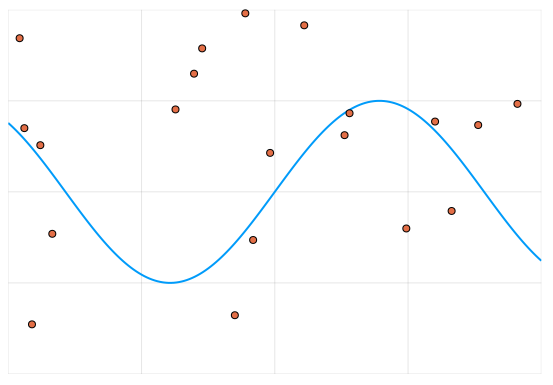}
			\caption{Left panel: histogram of the classification error. Middle panel: Initial training sample in the worst case. Right panel: Initial training sample in the best case.}
			\label{hist} 	
		\end{center}
\end{figure}
\begin{table}[h]
	\footnotesize
	\begin{center}
		\begin{tabular}{cccccc} 
			Min.  & 1st Qu.  & Median &   Mean &3rd Qu.&    Max. \\
			\hline
			0.0179& 0.0267 & 0.0365& 0.0458 & 0.0469 & 0.1554 
		\end{tabular}
		\caption{Summary of the classification error rate over 50 repetitions.}\label{errej1}
	\end{center}
\end{table}


Next, we compare the misclassification error of the semi-supervised methods introduced in this work with that of some  supervised classification algorithms trained with $\mathcal D^n$ to label $\mathcal X_l$.  $k$-nn is, naturally, the first method to be considered. As often happens in the presence of a tuning parameter, the choice of $k$ may impact on the performance of the procedure. In particular, in the present scenario, $k$ will be chosen on the base of  the training set $\mathcal D^n$, with $n=20$, which may turn in an  unstable recipe   to select $k$. 
To analyze the distribution of $\hat k$ in such a situation, we  generated 1000 samples of $\mathcal{D}^{20}$ and  for each of them we computed $\hat{k}$. The results are shown in Figure \ref{hist2}.
\begin{figure}[h]
	\begin{center}
		\includegraphics[scale=.2]{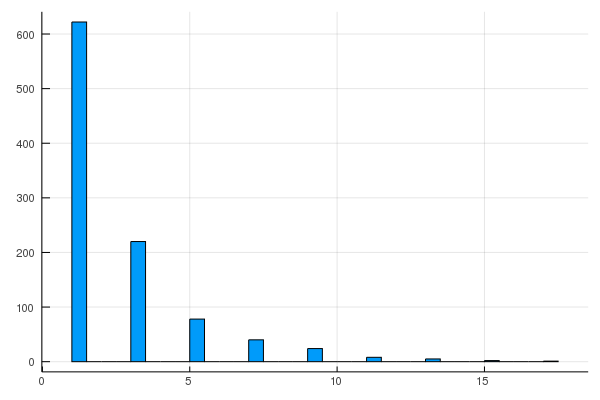}
		\caption{Histogram of values of $k$ chosen by cross validation procedure, over 1000 replications, for model one with $n=20$}
		\label{hist2} 	
	\end{center}
\end{figure}
The three more frequent values of $\hat k$ ($\hat{k}=1,3$ and $5$) were used to classify $\mathcal X_{2400}$  with a $k$-nn method trained with $\mathcal D^{20}$. 
The mean error rates  among 1000 replications  are given in Table \ref{errk}. It is worth to mention that the error rates of SSL and $k$-nn are computed with 50 and 1000 replications, respectively. This is due to computational demands of each procedure.

\begin{table}[h]
	\begin{center}
		\footnotesize
		\begin{tabular}{l|cccccc}
					& Min.  & 1st Qu.& Median & Mean  & 3rd Qu. & Max.  \\ \hline
			$k=1$   & 0.026 & 0.076  & 0.097  & 0.104 & 0.123  & 0.289 \\ 
			$k=3$   & 0.034 & 0.091  & 0.114  & 0.124 & 0.147  & 0.368\\
			$k=5$   & 0.052 & 0.108  & 0.129  & 0.142 & 0.163  & 0.448\\ \hline
		\end{tabular}
		\caption{Summary over 1000 replications of the misclassification error rate to classify $\mathcal{X}_{2400}$ for  the 3 more frequent values of $\hat{k}$.}\label{errk}
	\end{center}
\end{table}

\vspace{-.2cm}

Finally, we kept $n=20$ and vary  $l$, choosing $l=50\times j $, with $j=1,\ldots, 60$. At  each of the $60$ steps, $50$ new unlabelled data are included to conform the set $\mathcal X_l$. Four competitors were considered: $k$-nn with $k$ chosen by cross validation, support vector machine (SVM,  using the package \textsf{LIBSVM} in julia 1.2), both of them trained with $\mathcal D^n$,  and our SSL proposal with $h1=1.7(\log(l)/l)^{1/4}$ and $h2=0.7\times h1$. The whole procedure is repeated $50$ times. In figure \ref{sslinc} we plot the median, mean, 0.25, and 0.75 quantile of the misclassification errors respectively. As expected, the misclassification errors for $k$-nn and SVM remain mainly constant, while the misclassification error of the SSL algorithm decreases with $l$. Observe that for $l$ around 500 the SSL algorithm outperforms its competitors.

\begin{figure}[h]
	\begin{center}
		\includegraphics[scale=.24]{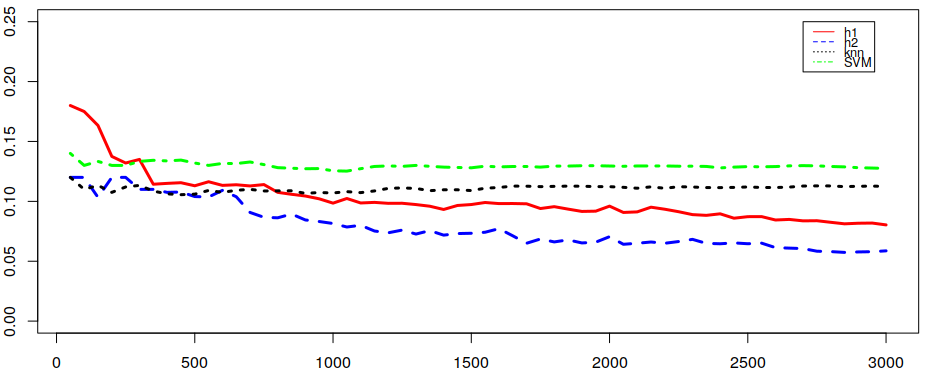}
		\includegraphics[scale=.24]{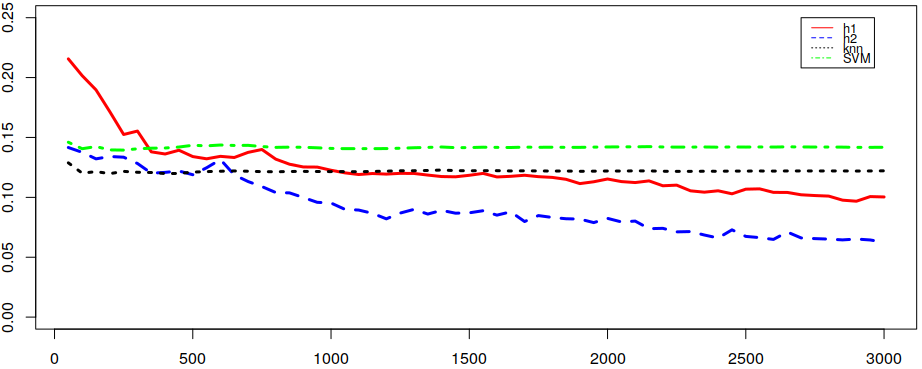}
		\includegraphics[scale=.24]{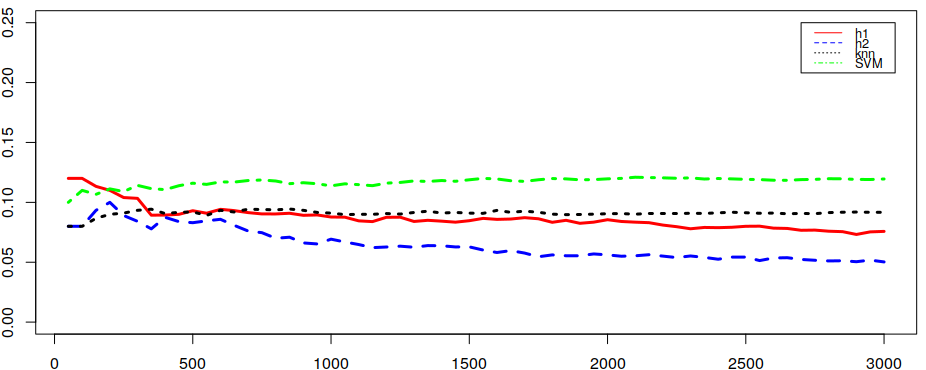}
			\includegraphics[scale=.24]{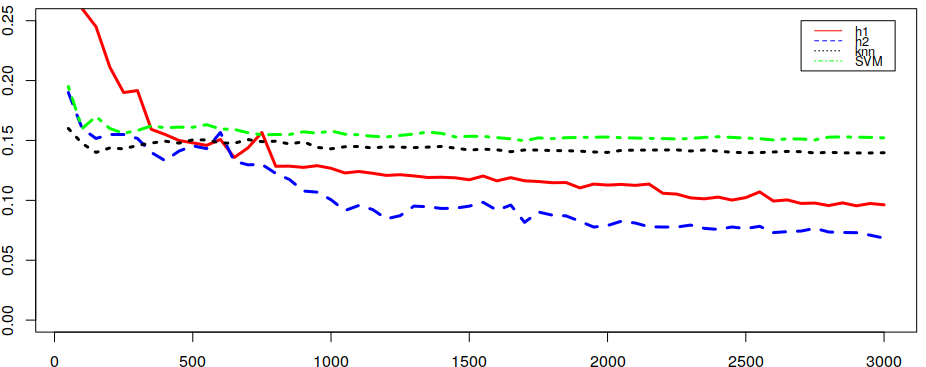}
		\caption{First row: Left, the medians  of the  misclassification errors. Right, means of the misclassification errors. 
			Second row:  Left, quantile 0.25 of the misclassification errors. Right, quantile 0.75 of the misclassification errors }
		\label{sslinc} 	
	\end{center}
\end{figure}

Figure \ref{fig2} exhibits the labels assigned by four different methods to a fixed realization of both $\mathcal X_l$ and $\mathcal D^n$. In the first row we show the labels assigned by the algorithm (with $h_l=0.15$), and the fast version of it. In the second row, the labels assigned by  $k$-nn, with  $k$=7 and SVM. The classification error  rates corresponding to each method  are $0.035$,  $0.06$, $0.14$ and $0.13$, respectively. 

\begin{figure}[htbp]
	\begin{center}			
		\includegraphics[scale=.35]{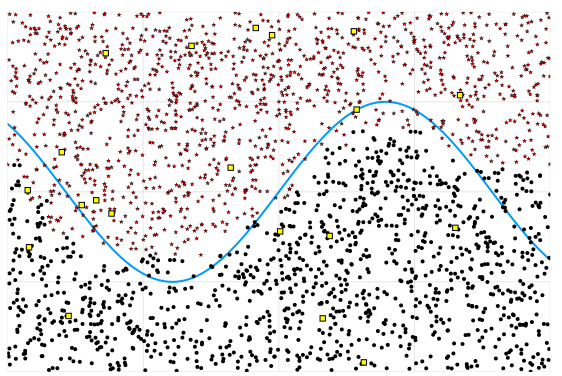} 
					\includegraphics[scale=.35]{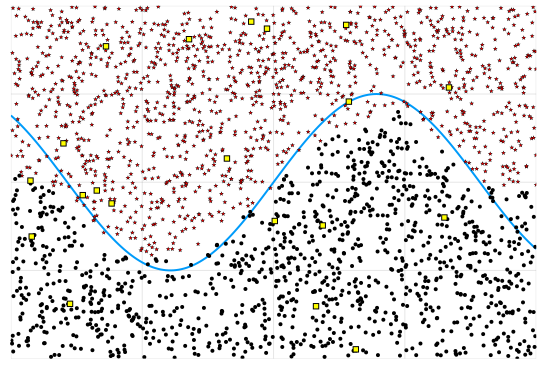} \\
	\includegraphics[scale=.35]{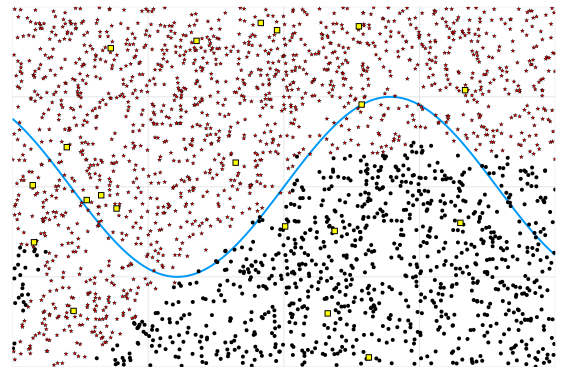} 	
			\includegraphics[scale=.36]{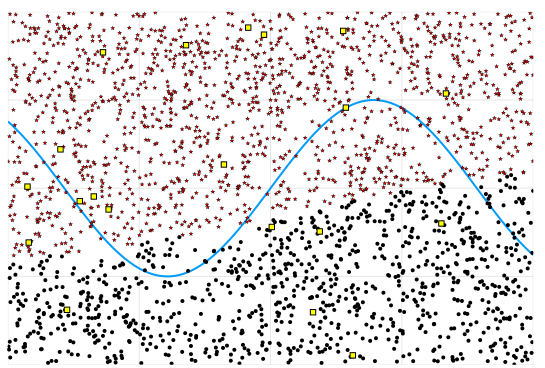} 
 \caption{Labels assigned by four different methods to a fixed realization of both $\mathcal X_l$ and $\mathcal D^n$. Red stars are points labelled as $1$ while black dots are labelled as $0$. The initial training sample $\D^n$ is represented as yellow squares. First row: left panel, faster version of the algorithm, presented in Section \ref{faster}, using  a $N$-grid with $N=21$ (distance 0.1 between points in each dimension).   Right panel output of the algorithm ran with bandwidth $h_l=0.15$. Second row: left panel,  labels assigned by  $k$-nn, with  $k$=7. Right panel  labels assigned by  SVM.}
		\label{fig2} 
\end{center}
\end{figure}
\subsection{A second example using simulated data }

To generate the data  consider two bi-variate normal random vectors  $Z_0\sim N(\mu_0,\Sigma)$ and $Z_1\sim N(\mu_1,\Sigma)$. Let  $Y\sim Bernoulli( 0.5)$. The conditional distribution of $X$ given $Y=y$, for $y=0,1$,  is given by $X\mid Y=y \sim Z_y \mid \|Z_y-\mu_y\|<1.5 $.

 We consider two cases: $\mu_0=(1.5,1.5)$, $\mu_1=(0,0)$ (see Figure \ref{gau} left) and $\mu_0=(1.2,1.2)$, $\mu_1=(0.0)$ (see Figure \ref{gau} right); in both cases $\Sigma=\text{diag}(0.6,0.6)$. 
In the first case the Bayes error  is $0.025$ and in the second one is $0.067$. 

We generate $\X_l=(X_1,\dots,X_l)$ iid, with $X_i$ distributed as $X$, and sample size $l=2000$.  In each replication, we used $\mathcal{D}^n=\{((0,0),1),((1.5,1.5),0)\}$ and bandwidth  $h=0.4$  to run the algorithm.

 The average of the computational time as well as the  error rate over 50 replications are reported in Table \ref{tabgau}, for different grid sizes. As it is shown in Table \ref{tabgau}, there is a trade-off between computation time and efficiency. However, if the cell sizes of the grid are reasonably small (as in the first column of Table \ref{tabgau}), the misclassification errors are essentially the same, while the computational time decreases. The simulation  was performed in \textsf{Julia} 1.0.1, running on an Intel i7-8550U.

\begin{figure}[htbp]
	\begin{center}
		\includegraphics[scale=.25]{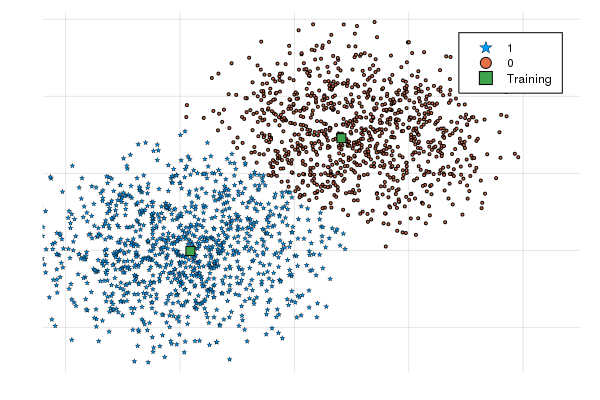} 
		\includegraphics[scale=.25]{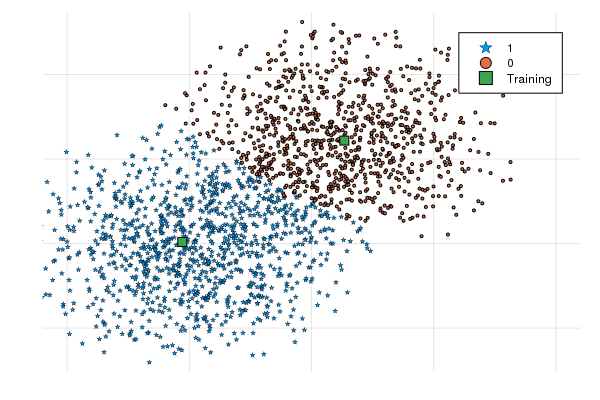} 
		\caption{The two populations of bi-variate truncated Gaussian distributions.}
		\label{gau}
	\end{center}
\end{figure}

\vspace{-1.2cm}
\begin{table}[h]
	\footnotesize
	\begin{center}
		\begin{tabular}{cc|cc|cc}
			\multicolumn{2}{c|}{Without Grid} & \multicolumn{2}{c|}{Grid step 0.1} & \multicolumn{2}{c}{Grid step 0.15} \\ \hline
			                                      \multicolumn{6}{c}{First Case}                                        \\ \hline
			Time  &           Error           & Time &            Error            & Time  &           Error            \\ \hline
			4.2s  &          0.0323           & 2.8s &            0.043            & 1.1s  &           0.046            \\ \hline
			                                      \multicolumn{6}{c}{Second Case}                                       \\ \hline
			Time  &           Error           & Time &            Error            & Time  &           Error            \\ \hline
			5.15s &           0.084           & 2.7s &            0.10             & 0.98s &           0.117
		\end{tabular}
		\caption{Average of the computation time  and miss-classification  errors over 50 replications.}\label{tabgau}
	\end{center}
\end{table}

\subsection{A real data example}
We consider the well known Isolet data set of speech features from the UCI Machine Learning Repository \cite{an:07}, comprising 617 attributes associated with the English pronunciation of the 26 letters of the alphabet. The data come from 150 people who spoke the name of each letter twice. There are three missing data, not considered in the study. Feature vectors include: spectral coefficients, contour features, sonorant features,
pre-sonorant features, and post-sonorant features, and are described in \cite{fc:91}. The spectral coefficients account for 352 of the features. The exact order of appearance of the features is not known.

We  apply the semi-supervised algorithm to the binary problem given by the  E-set comprising the letters $\{b, c, d, e, g, p, t, v, z\}$ and the R-set with the remaining letters except for the letters $\{m,n\}$, starting with a small labelled data set of 10 elements from each group. Then $\mathcal{D}^n$ consists of $20$ data.

To pre-process the data, we first removed the first repetition of every letter. Next, we kept only those data whose nearest neighbour is at a distance smaller than a threshold (the value $8$ was selected to reduce the misclassification error, and to reduce the computational time, in order repeat it 100 times). This pruning procedure reduced the sample $\mathcal{X}_l$ to $2171$ data. To study how the misclassification error varies with respect to the training sample, we randomly chose a training sample 100 times.  We compared with $k$-nn  choosing $k$ by cross validation, using $500$ replicates. A summary of the misclassification error rates is shown in Figure \ref{err} right, while the density of the errors of SSL is shown in Figure \ref{err} left.\\

\begin{minipage}[l]{\textwidth}
		\begin{minipage}[c]{0.3\textwidth}
			\tiny
			\includegraphics[scale=.15]{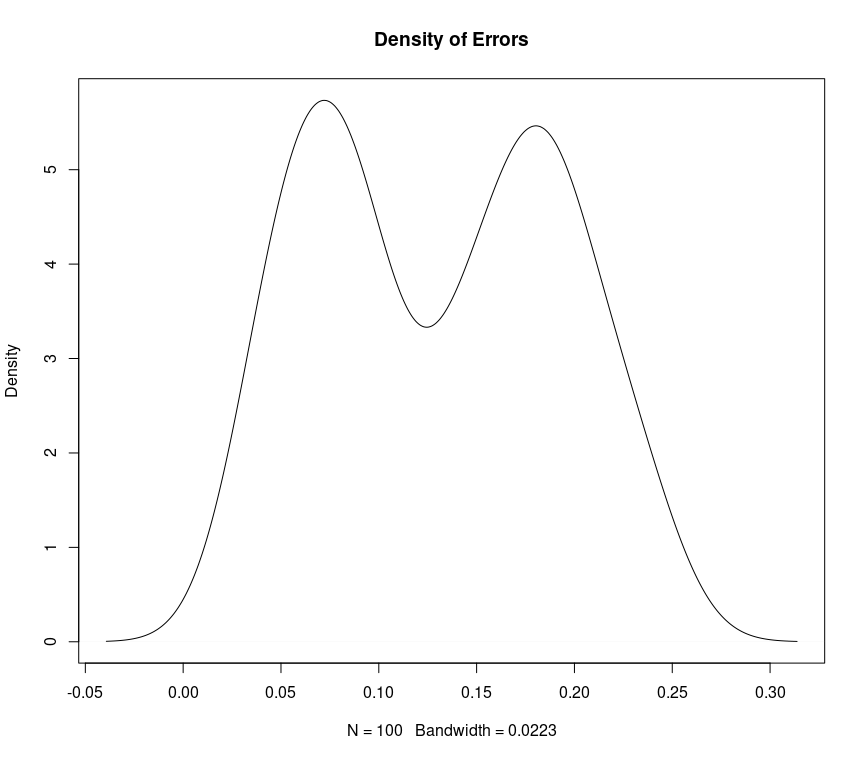} 
		\end{minipage} \begin{minipage}[r]{0.6\textwidth}
			\footnotesize
			\begin{tabular}{l|cccccc}
				       & Min.  & 1st Qu. & Median & Mean  & 3rd Qu. & Max.  \\ \hline
				SSL    & 0.028 &  0.076  & 0.139  & 0.130 &  0.184  & 0.247 \\
				$k$-nn & 0.048 &  0.124  & 0.149  & 0.144 &  0.168  & 0.216 \\ \hline
			\end{tabular}

\end{minipage}
			\captionof{figure}{Left: Density estimator of the errors of the SSL. Right: Summary of the missclassification error rate over 100 replications for SSL and 500 for $k$-nn with $k$ chosen by cross validation.}\label{err}
\end{minipage}

\section{Some remarks regarding the assumptions}\label{assump} 
 We discuss briefly the set of assumptions considered. Firstly, we would like to point out that the results we are looking at are quite strong, since the training sample is frozen  at a small fixed size $n$ and the asymptotic is on $l$ the unlabeled data size. These results should not be  misinterpreted. Without these hypotheses, the semi-supervised classification methods may work better than the classical supervised classification methods but the consistency will not be verified if the size $n$ of the training sample remains fixed. 
\begin{itemize}
	\item [1)] In order for an algorithm to work for the semi-supervised classification problem, the initial training sample $\mathcal{D}^n$ (whose size does not need to tend to infinity) must be well located. We require that $\mathcal{D}^n=(\X^n,\Y^n)$ satisfies $Y^i=g^*(X^i)$ for all $i=1,\dots,n$, which is a quite mild hypothesis. In many applications, a stronger condition can be assumed. For instance, if the two populations are sick or healthy,  the initial training sample can be chosen as the set of individuals for whom the covariate $X$ ensures the condition on the patient, that is, $\mathbb{P}(Y=1|X)=1$ or $\mathbb{P}(Y=1|X)=0$. On the other hand, if the initial training sample is not well located, then any algorithm  might classify almost all observations wrongly. Indeed, consider the case where the distribution of the population with label 0 is $N(0,1)$ and the other is $N(1,1)$. This will be the case if we start for instance with the pairs $\{(0.4,1),(0.6,0)\}$.\\
The effect of the initial training sample $\mathcal{D}^n$ is illustrated in the real-data example where the misclassification error varies between 0.028 and 0.247 by changing at random $\mathcal{D}^n$.
	\item[2)] The connectedness of $I_0$ and $I_1$ is also critical. In a situation like the one shown in Figure \ref{fig1}, the points in the connected component for which there is no point in $\D^n$ (represented as squares) will be classified as the circles by the algorithm. However, if $I_0$ and $I_1$ have a finite number of connected components and there is at least one pair $(X^i,Y^i)\in \D^n$ in each of them with $g^*(X^i)=Y^i$, it is easy to see that the algorithm will be consistent. 
	\item[3)] The uniform kernel assumption can be replaced by any regular kernel satisfying $c_1 I_{B(0,1)}(u)\leq K(u) \leq c_2 I_{B(0,1)}(u),$ for some positive constants $c_1, c_2$, and the results still hold. 
	\item [4)]   In Proposition \ref{mariela} we assume that $P_X$ has a continuous density $f$ with compact support $S$. If that is not the case, it is possible to take a large enough compact set $S$ such that $P_X(S^c)$ is very small and therefore just a few data from $\X^l$ is left out.
	\item[5)] The following example shows that H5 is necessary for consistency. Indeed, suppose that $U_1:=X|Y=1\sim U([a,1])$ and $U_0=X|Y=0\sim U([0,a])$ with $a=P(Y=0)$, then for all $a\in [0,1]$, $P_X=aU_0+(1-a)U_1\sim U([0,1])$. Unless the training sample $\D^n$ contains two points $(X_1,0)$ and $(X_2,1)$ with $X_1$ and $X_2$ very close to $a$, semi-supervised methods will fail. Regardless of the value of $a$, the classes 0 and 1 are indistinguishable since the joint distribution is in all cases $U[0,1]$. 
	
	Moreover, there is no consistent semi-supervised algorithm for $n$ fixed. To see this, consider $(X_1,Y_1),\dots,(X_n,Y_n)$ a training sample in $[0,1]$ with fixed size $n$. 
	Let us denote $X_m=\min\{X_i: (X_i,1)\in \mathcal{D}^n \}$, $X^M=\max\{X_i: (X_i,0)\in \mathcal{D}^n\}$, $\mathcal{X}_l\cap (X_m,X^M)=\{X_{i_1},\dots X_{i_k}\}$ and $\tilde{Y}_{i_1},\ldots,\tilde{Y}_{i_k}$ the labels assigned by any algorithm. Then if $\sum_{j=1}^k \tilde{Y}_{i_k}>k/2$, conditioned to the training sample $\mathcal{D}^n$, if we choose $a=X^M$  we will miss-classify at least $k/2$ data-points, and if $\sum_{j=1}^k \tilde{Y}_{i_k}\leq k/2$, $a=X_m$ we will do the same. 
	
\end{itemize}
\begin{figure}[h]
	\begin{center}
		\includegraphics[scale=.18]{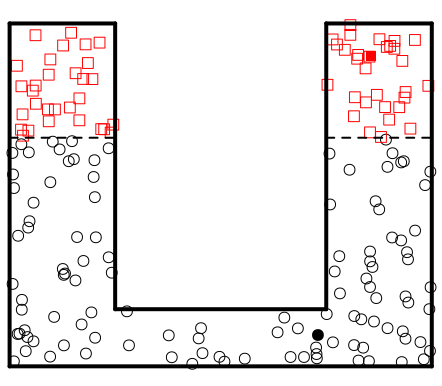} 
		\caption{The points labelled as $0$ are represented with squares while the points labelled as 1 are represented with circles. Filled points belong to $\D^n$.}
		\label{fig1}
	\end{center}
\end{figure}

\section{Concluding remarks}

In this paper we address the problem of semisupervised learning. We propose a simple algorithm and analize its asymptotic behaviour. 

The focus is on understanding when and why SSL works when the training sample is small and frozen, and the asymptotic is in the sense of the strong formulation given in equation \eqref{consistency}, where the limit is only on the size $l$ of the unlabelled data-set.

From the discussion on Section \ref{assump} it follows that SSL will only work under the  almost necessary hypotheses we have assumed.

The first simulation example shows the behaviour of our algorithm. In  particular Figure \ref{hist}  and Table \ref{errej1} exhibit the effect of the initial training sample $\mathcal{D}^n$ on the output of the procedure, which is related to assumption H8. 
We also study the effect of increasing the size of the unlabelled data, where we compare with two well known competitors: $k$-nearest neighbours and support vector machine. Our SSL proposal outperforms the competitors if $l$ is larger than $500$.

The second simulated example studies the trade-off between computation time and efficiency. Lastly, on the real-data example we challenge our algorithm by comparing with the $k$-nearest neighbour rule performed on $\mathcal{D}^n$. In this case the results are only slightly better.

\section*{Appendix A}\label{app}

 \textit{Proof of Proposition  \ref{prop0}.}\\
 Observe that $\mathbb{P}\big(g_i(\X_l)\neq Y_i \mid \X_l\setminus X_i\big)\geq \mathbb{P}(g^*(X_i)\neq Y_i) $, for $i=1,\ldots, l$. Thus, 
	$$
	\mathbb{E}\Big(\mathbb{I}_{g_i(\X_l)\neq Y_i}\Big)=\mathbb{P}(g_i(\X_l)\neq Y_i)
	=\mathbb{E}\Big(\mathbb{P}\big(g_i(\X_l)\neq Y_i|\X_l\setminus X_i\big)\Big)\geq \mathbb{P}(g^*(X_i)\neq Y_i),
	$$
	and therefore, $L({\mathbf{g}_l})=\mathbb{E}\Big(\frac{1}{l}
	\sum_{i=1}^l
	I_{g_i(\X_l)\neq Y_i}\Big)\geq \mathbb{P}(g^*(X_i)\neq Y_i),$ 
	showing that $L({\mathbf{g}_l})\geq \mathbb{P}(g^*(X)\neq Y)$,   for any $\mathbf{g}_l=(g_1,\ldots, g_l)$. The lower bound is attained by choosing  the $i$th coordinate of 
$ \mathbf{g}_l$ equal to $g^*(X_i)$. Moreover, the accuracy  of $\mathbf{g}^\ast_l$ equals that of a single coordinate; namely $L(\mathbf{g}^\ast_l)= \mathbb{P}(g^*(X)\neq Y)=L^*$.  \\

 \textit{Proof of Proposition  \ref{alg2}.}\\
We will prove that if H1, H2 i) and H4  are satisfied, then 
$\mathcal I_{0,l}\cap \mathcal I_{1,l}\subset \mathcal F_l$. Combining this inclusion with  H3 we conclude that $\mathbb{P}(\mathcal F)=1$. 
To prove that 
$\mathcal I_{0,l}\cap \mathcal I_{1,l}\subset \mathcal F_l$, we will see that if 
\begin{equation}
\label{cubro_nuevo}
I_a\subseteq \bigcup_{X\in \X_{l}\cap I_a } B(X,h_l/2)\;,\quad a=0,1, 
\end{equation} 
all the elements of $\X_l$ are labeled by the algorithm.  To do so, note that, by   H4, there exists   $X_a^{\ast}$ in $\X^n$ such that 
$X_a^{\ast}\in I_a$, for $a=0,1$. We will now prove that the algorithm starts. 
Since  $X^*_1$ is in $I_1$ and \eqref{cubro_nuevo} holds with $a=1$, there exists $X_j^1\in \X_l\cap I_1$ with  $d(X_1^{\ast}, X_j^1)< h_l$.   In particular, $d(\X^n, X_j^1)<h_l$ and so $ X_j^1\in \unclass_0(h_l)$. This guarantees that $\unclass_0(h_l) \not=\emptyset$ and hence the algorithm can start.

Assume now that we have classified $j< l$ points of $\X_l$. We will prove that  there exists at least one point satisfying the iteration condition required at step $j+1$: $\unclass _{j}(h_l)\not=\emptyset $. By H1 we can assume that $\unclass_j=\unclass_j\cap (I_0\cup I_1)$.  Take $a$ such that $\unclass_j\cap I_a\not=\emptyset$.   We will consider now two possible cases: (i) if $\X_l\cap I_a\cap \unclass_j^c=\emptyset$, then $\X_l\cap I_a=\X_l\cap I_a\cap \unclass_j $ and so, by \eqref{cubro_nuevo}, $X^*_a\in B(X,h_l/2)$ 
 for some $X\in \X_l\cap \unclass_j$. Since $X^*_a$ is in $\Z_j$ and $X\in \unclass_j$, we conclude that $X\in \unclass_j(h_l)$. Assume now that (ii)
 $\X_l\cap I_a\cap \unclass_j^c\not=\emptyset$.
 Since $I_a$ is connected and  \eqref{cubro_nuevo} holds,  the union of 	 $ B(X, h_l/2)$, with $X \in \X_l\cap I_a$,  is also a connected set and, therefore,
 $$\Bigg(\bigcup_{X\in \X_{_l}\cap I_a \cap \unclass_{j}^c} B(X,h_l/2)\Bigg)\  \ \bigcap \    \Bigg(\bigcup_{X\in \X_{l}\cap I_a \cap \unclass_j} B(X,h_l/2)\Bigg)\neq \emptyset.$$
 Finally, take $X\in \X_{_l}\cap I_a \cap \unclass_{j}^c$  
 and  $\tilde X\in \X_{_l}\cap I_a \cap \unclass_j$ such that 
 $B(X, h_l/2)\cap B(\tilde X,h_l/2)\not=\emptyset$ to conclude that $\tilde X\in \unclass_j(h_l)$. \\

 \textit{Proof of Lemma  \ref{lem1}.}\\
By H1, we can assume that $\eta(X)\neq 1/2$ for all $X\in \X^n\cup\X_l$.
	Assume first  that $\eta(X_{j_{bad}})>1/2$, that is, $X_{j_{bad}}\in I_1$, $\tilde Y_{j_{bad}}=0$, and all the points labelled up to the step $j_{bad}-1$  by the algorithm are well classified. Now, suppose by contradiction $X_{j_{bad}}\not\in B_1^{h_l}$,  which means  that $X_{j_{bad}}\not\in B(I_0, h_l)$ and thus,  $B(X_{j_{bad}}, h_l)\cap I_0=\emptyset$. This implies that 
	$g^*(X)=1 $ for all $X\in (\X^n\cup  \{X_{i_1},\dots,X_{j_{bad} -1} \})\cap B(X_{j_{bad}},h_l)$, contradicting the label assigned to $X_{j_{bad}}$ according to the majority rule that is used by the algorithm.
	Thus, $B(X_{j_{bad}}, h_l)\cap I_0\not=\emptyset$, and so $X_{j_{bad}}\in B_1^{h_l}$. Analogously, if $\eta(X_{j_{bad}})<1/2$, we deduce that $X_{j_{bad}}\in B_0^{h_l}$. \\

	 \textit{Proof of Lemma  \ref{lemaux}.}\\
	 	Given $\delta<\delta_1$, choose $\varepsilon$ such that $\gamma(\delta)-2\varepsilon >0$, for $\gamma(\delta)$ introduced in H6.  We will prove $\mathcal S_l^\varepsilon= \{\sup_{u \in S}|\hat f_l(u)-f(u)|<\varepsilon\}$ is included in $\mathcal \mathcal \mathcal V_l^\delta$ 	as far as $h_l<\delta$  and therefore, from   \eqref{conv_unif}, we conclude that $\mathbb{P}(\mathcal V^\delta )=1$. 
	
	Now, note that on $\mathcal S_l^\varepsilon$, we get that $f(u)-\varepsilon < \hat f_l(u) < f(u)+\varepsilon,$
	and so, on $\mathcal S_l^\varepsilon$, for  $a\in A_0^{\delta}\cup A_1^{\delta}$ and $b \in  B_1^{h}\cup B^{h}_0$,  $\hat f_l(b)< f(b)+\varepsilon <f(a)-\gamma  +\varepsilon < \hat f_l(a)+2\varepsilon- \gamma.$ Thus, on $\mathcal S_l^\varepsilon$, 
	$$\sup_{b\in  B_0^{h_l}\cup B_1^{h_l}}\hat f_l(b) \;\leq\;  \inf_{a\in A_0^\delta\cup A_1^\delta} \hat f_l(a)+2\varepsilon-\gamma<  \inf_{a\in A_0^\delta\cup A_1^\delta} \hat f_l(a),  $$
	when $2\varepsilon -\gamma<0$.  This proves that $\mathcal S_l^\varepsilon\subseteq \mathcal \mathcal F_l^\delta$, for $l$ such that $8h_l<\delta$.  
	
	\bigskip

		 \textit{Proof of Lemma  \ref{lem2}.}\\
When $j_{bad}=\infty$, $\Z_{j_{bad}-1}=\X_n\cup \X_l$. This fact implies that,  on  the event   $\mathcal F_l\cap\mathcal B_l^c$,  the following identity holds: $\X_l\cap  (\Z_{j_{bad}-1})^c=\emptyset $. Thus, to prove \eqref{prop_10_new}, 	 we need  to show that, for $a=0,1$
$\mathcal F_l \cap \mathcal A_{a,l}^\delta \cap \mathcal V^\delta_l\cap \mathcal B_l \;\subset\; \left\{\X_l\cap A_a^\delta \cap (\Z_{j_{bad}-1})^c=\emptyset  \right\}.$
We will argue by contradiction, assuming that there exists $\omega \in \mathcal F_l \cap \mathcal A_{a,l}^\delta \cap \mathcal V^\delta_l\cap \mathcal B_l $ 
for which $\emptyset\not= \X_l\cap A_a^\delta \cap (\Z_{j_{bad}-1})^c=\{W_1, \ldots, W_m\}$.  Invoking H8,  $\X^n\subseteq \Z_{j_{bad}-1}$ and there exists $X_a^\ast \in A_a^\delta \cap \X^n$. These facts guarantee that $X_a^\ast \in A_0^\delta \cap \Z_{j_{bad}-1}$, and since we are working on $\mathcal A_{a,l}^\delta$, we get that
\begin{equation}
\label{cubro_lemma}
X_a^\ast \in A_a^\delta \subseteq \bigcup_{X\in \X_{l}\cap A_a^\delta } B(X,h_l/2)\quad\hbox{and}\quad X_a^\ast \in\Z_{j_{bad}-1}. 
\end{equation}
Next, we will argue that there exist $W^\ast\in \{W_1, \ldots, W_m\}$  such that \break $d(W^\ast ,\Z_{j_{bad}-1})<h_l$. To do so, consider the following two cases: 
\begin{itemize}
\item [(i)] $\X_l \cap A_a^\delta \cap\Z_{j_{bad}-1} =\emptyset$. In such a case,  from \eqref{cubro_lemma} we get that $A_a^\delta$ can be covered by balls  centered at $\{W_1, \ldots, W_m\}$  and, since $X_a^\ast\in A_a^\delta$, $X_a^\ast\in B(W^\ast, h_l/2)$ for  some $W^\ast \in \{W_1, \ldots, W_m\}$. Therefore, $d(X_a^\ast, W^\ast)<h_l$. Recalling that, as stated in \eqref{cubro_lemma}, $X_a^\ast \in\Z_{j_{bad}-1}$, we conclude that  \break $d(W^\ast ,\Z_{j_{bad}-1})<h_l$.

\item [(ii)] Assume now that $\X_l \cap A_a^\delta \cap \Z_{j_{bad}-1} \not =\emptyset$.  Since $A_a^\delta $ is connected, the union of balls given in \eqref{cubro_lemma} is connected, and then,
\begin{equation*}
\Bigg\{\bigcup_{X\in \X_{l}\cap A_a^\delta\cap  \Z_{j_{bad}-1} } B(X,h_l/2)\Bigg\} \quad \bigcap \quad
\Bigg\{\bigcup_{1\leq i\leq m} B(W_i,h_l/2)\Bigg\}\quad \not=\quad \emptyset. 
\end{equation*}
Thus, there exist $X\in \Z_{j_{bad}-1}$ and $W^\ast\in \{W_1, \ldots, W_m\}$ with $d(X^\ast, W^\ast)<h_l$, which implies that  
$d(W^\ast ,\Z_{j_{bad}-1})<h_l$.
\end{itemize}
 To finish the proof, we will show that such a $W^\ast$ should have been chosen by the algorithm to be labelled before $X_{j_{bad}}$,  which implies that $W^\ast\in \mathcal{Z}_{j_{bad}-1}$, contradicting that $W^\ast\in (\mathcal{Z}_{j_{bad}-1})^c$. This contradiction show that
no such $W^\ast$ exists, as announced.  
Since $d(W^\ast ,\Z_{j_{bad}-1})<h_l$, we get that 
$W^\ast \in \unclass_{j_{bad}-1}(h_l)$, the set of candidates to be labelled by the algorithm at step $j_{bad}$. Indeed, since $W^\ast \in A_a^\delta$ and $h<\delta$, $B(W^\ast, h_l) \subseteq I_a$. Thus, $\hat{\eta}_{j_{bad}-1}(W^\ast)=a$ implying that $W^\ast$ attains the maximum stated in \eqref{ji}.  Invoking now Lemma \ref{lemaux}, since $W^\ast \in A_a^\delta$ while  $X_{j_{bad}}$ is in $B_0^h\cap B_1^h$ (see Lemma \ref{lem1}), 
we know that  $\#\{\X_l\cap B(W^\ast,h_l)\}\geq \#\{\X_l\cap B(X_{j_{bad}},h_l)\}$; thus, $W^\ast$ should have been chosen before $X_{j_{bad}}$. This conclude the prof of the result. 

			\begin{figure}[h]
				\begin{center}
					\includegraphics[scale=.5]{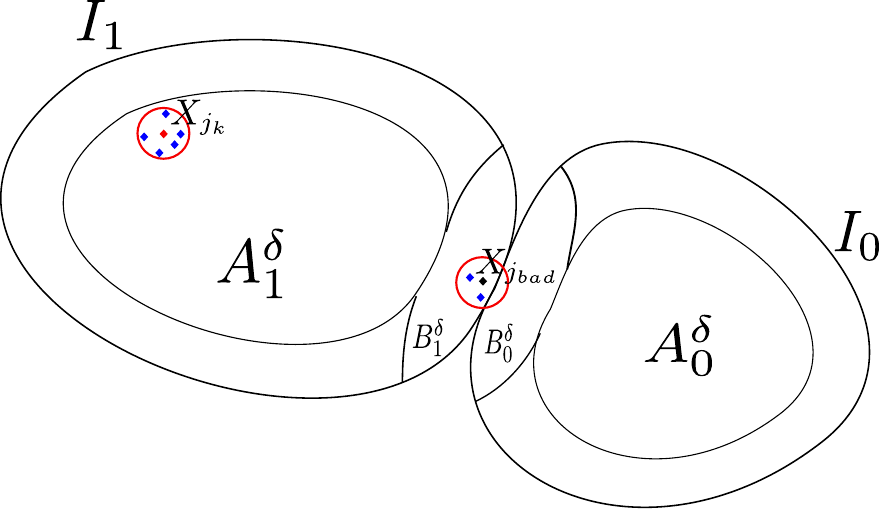} 
					\caption{We show: in black $X_{j_{bad}}$, in red  $X_{j_k}$, in blue we represent the points of $\mathcal{X}_{l_k}$ belonging to $B(X_{j_k},h_{l_k})$ and $B(X_{j_{bad}},h_{l_k})$.}
					\label{IAB}
				\end{center}
			\end{figure}

\textit{Proof of Theorem \ref{teoconst}}

Recall that $g_{{n,l,r(i)}}(\X_l)$ denotes the label assigned by the algorithm to the observation$ X_i\in \X_l$.  	The empirical mean accuracy of classification satisfies 
\begin{align*}
\frac{1}{l}\sum_{i=1}^l \mathbb{I}_{g_{{n,l,r(i)}}(\X_l)= Y_i}\geq & \frac{1}{l}\sum_{i=1}^{l} \mathbb{I}_{g_{{n,l,r(i)}}(\X_{l})= Y_i}\;\mathbb{I}_{g^*(X_i)=Y_i}\;\mathbb{I}_{A_0^\delta\cup A_1^\delta}(X_i)\\
=& \frac{1}{l}\sum_{i=1}^{l} \mathbb{I}_{g_{{n,l,r(i)}}(\X_{l})= g^*(X_i)} \;\mathbb{I}_{g^*(X_i)=Y_i}\;\mathbb{I}_{A_0^\delta\cup A_1^\delta}(X_i). 
\end{align*}
Consider  $\mathcal T_l^\delta=\mathcal F_l\cap \mathcal A_{0,l}^\delta \cap \mathcal A_{1,l}^\delta\cap  \mathcal V_l^\delta$, and $\mathcal{T}^\delta =\bigcup_{l_0}\bigcap_{l\geq l_0}\mathcal T_l^\delta$.  Combining the results obtained in  Proposition \ref{alg2} and  Lemma \ref{lemaux} with 
condition  H5,  we conclude  that $\mathbb{P}(\mathcal T^\delta)=1$, for $\delta<\min\{\delta_0, \delta_1,\delta_2\}$. By \eqref{eq4},  on $\mathcal T_l$, we have that $\mathbb{I}_{g_{{n,l,r(i)}}(\X_{l})= g^*(X_i)}\geq \mathbb{I}_{A_0^\delta\cup A_1^\delta}(X_i)\quad \text{ for all }i=1,\dots,l,$
and therefore
\begin{eqnarray*}
	\frac{1}{l}\sum_{i=1}^{l} \mathbb{I}_{g_{{n,l,r(i)}}(\X_{l})= g^*(X_i)} \;\mathbb{I}_{g^*(X_i)=Y_i}\;\mathbb{I}_{A_0^\delta\cup A_1^\delta}(X_i)\geq 
	\frac{1}{l}\sum_{i=1}^{l} 
	\mathbb{I}_{g^*(X_i)=Y_i} \;\mathbb{I}_{A_0^\delta\cup A_1^\delta}(X_i).
\end{eqnarray*}
Then, on $\mathcal T^\delta$, we have that $\liminf_{l\rightarrow \infty}	\frac{1}{l}\sum_{i=1}^l \mathbb{I}_{g_{{n,l,r(i)}}(\X_l)= Y_i}\geq \mathbb{P}\{g^*(X)=Y, X\in A_0^\delta\cup A_1^\delta \}$	and so 
\begin{eqnarray}
\label{liminf}
\liminf_{l\rightarrow \infty}	\frac{1}{l}\sum_{i=1}^l \mathbb{I}_{g_{{n,l,r(i)}}(\X_l)= Y_i} \geq \mathbb{P}\{g^*(X)=Y\} \quad a.s.
\end{eqnarray}
On the other hand, 
\begin{eqnarray*}
	\frac{1}{l}\sum_{i=1}^l \mathbb{I}_{g_{{n,l,r(i)}}(\X_l)= Y_i}=
	\frac{1}{l}\sum_{i=1}^l \;\mathbb{I}_{g_{{n,l,r(i)}}(\X_l)= Y_i} \mathbb{I}_{g^\ast(X_i)=Y_i}\\ + \frac{1}{l}\sum_{i=1}^l \mathbb{I}_{g_{{n,l,r(i)}}(\X_l)= Y_i}\;\mathbb{I}_{g^\ast(X_i)\not=Y_i}
	\leq \frac{1}{l}\sum_{i=1}^l \mathbb{I}_{g^\ast(X_i)=Y_i} + \frac{1}{l}\sum_{i=1}^l \mathbb{I}_{g_{{n,l,r(i)}}(\X_l)\not= g^\ast(X_i)}\\
\end{eqnarray*}
From Lemma \ref{lem2}, on $\mathcal T_l^\delta$, 
$\mathbb{I}_{g_{{n,l,r(i)}}(\X_l)\not= g^\ast(X_i)}\leq
\mathbb{I}_{(A_0^\delta\cup A_1^\delta)^c}(X_{i})$, and therefore, on $\mathcal T^\delta$, 
$$ \limsup_{l\rightarrow \infty} \frac{1}{l}\sum_{i=1}^l \mathbb{I}_{g_{{n,l,r(i)}}(\X_l)= Y_i}\leq \mathbb{P}(g^\ast(X)=Y)+\mathbb{P}(X\not \in \{A_0^\delta\cup A_1^\delta \}) .$$
By H2 ii), the last term in the previous display  converges to zero when $\delta\rightarrow 0$, and thus 
\begin{equation}
\label{limsup}
\limsup_{l\rightarrow \infty} \frac{1}{l}\sum_{i=1}^l \mathbb{I}_{g_{{n,l,r(i)}}(\X_l) }\leq \mathbb{P}(g^\ast(X)=Y)\quad a.s.
\end{equation}
Combining  \eqref{liminf} and \eqref{limsup}  we deduce the announced convergence. The consistency defined in \eqref{consistency} follows from the Dominated convergence theorem.

\section*{Appendix B} \label{apB}

In this section we will prove that under $H2$, conditions $H3$, and $H6$ holds if we impose some geometric restrictions on $I_0$ and $I_1$. In order to make this Appendix self contained, we need  some geometric definitions and also include some results which will be invoked.

First we introduce the concept of Hausdorff distance.
Given two compact non-empty sets $A,C\subset{\mathbb R}^d$, the \it Hausdorff distance\/ \rm or \it Hausdorff--Pompei distance\/ \rm between $A$ and $C$ is defined by
$d_H(A,C)=\inf\{\eps\geq 0: \mbox{such that } A\subset B(C,\eps)\, \mbox{ and }
C\subset B(A,\eps)\}.$ 

Next, we define standard sets, according to \cite{cue97} (see also \cite{cue04}). 
\begin{definition} \label{st} A bounded set $S\subset \mathbb{R}^d$ is said to be standard with respect to a Borel measure $\mu$ if there exists $\lambda>0$ and $\beta>0$ such that $\mu\big(B(x,\eps)\cap S\big)\geq \beta \mu_L(B(x,\eps))\text{ for all }x\in S,\ 0<\eps\leq \lambda,$ where $\mu_L$ denotes the Lebesgue measure on $\mathbb{R}^d$.
\end{definition}

Roughly speaking, standardness prevents the set from having peaks that are too sharp. 

The following theorem is proved in \cite{cue04}). 

\begin{theorem} (\cite{cue04} ) \label{cuevas} 
	Let $Z_1,Z_2,\dots$ be a sequence of iid observations in $\mathbb{R}^d$ drawn from a distribution $P_Z$. Assume that the support $Q$ of $P_Z$ is compact and standard with respect to $P_Z$. Then 
	\begin{equation}\label{cuev}
	\limsup_{l\rightarrow \infty} \left(\frac{l}{\log(l)}\right)^{1/d}d_H(\mathcal{Z}_l, Q)\leq \left(\frac{2}{\beta \omega_d}\right)^{1/d}\quad \text{ a.s.},
	\end{equation}
	where $\omega_d=\mu_L(B(0,1))$, $\mathcal{Z}_l=\{Z_1,\dots,Z_l\}$, and $\beta$ is the standardness constant introduced in Definition \ref{st}.
\end{theorem}

\begin{remark} \label{rem1} Theorem \ref{cuevas} implies that, if we choose $\epsilon_l=C \left(\frac{\log(l)}{l}\right)^{1/d}$ with $C>(2/(\beta\omega_d))$, then $Q\subset \cup_{i=1}^l B(Z_i,\epsilon_l)$ for $l$ large enough. This in turn implies that if $Q$ is connected, $\cup_{i=1}^l B(X_i,\epsilon_l)$ is connected.
\end{remark}

As a consequence of Theorem \ref{cuevas}, 
we get the following covering property that will be used to prove  Proposition 2 and H5 alone Proposition \ref{mariela}.

\begin{lemma} \label{lemcub} Let $X_1,X_2,\dots$ be a sequence of iid observations in $\mathbb{R}^d$ drawn from a distribution $P_X$ with support $S$.
	Let $Q\subset S$, be  compact and  standard with respect to $P_X$ restricted to $Q$,  with  $P_X(Q)>0$. Consider $(h_l)_{l\geq 1}$ such that $h_l\rightarrow 0$ and $lh_l^d/\log(l)\rightarrow \infty$. Then, with probability one, for $l$ large, $Q \subset  \bigcup_{X\in \X_{l}\cap Q } B(X,h_l/2),$	where $\mathcal{X}_l=\{X_1,\dots,X_l\}$.
\end{lemma}

\begin{proof}
	 We need to work with $\X_l$ restricted to $Q$, in order to do that, consider  the sequence of stopping times defined by $\tau_0\equiv 0, \ \tau_1=\inf\{l:X_l\in Q\},  \  	\tau_j=\inf\{l\geq \tau_{j-1}:X_l\in Q\},$  and   the sequence of visits to $Q$ given by $Z_j:=X_{\tau_j}$.
	 	Then, $(Z_j)_{j\geq 1} $ are iid, distributed  as $X\mid (X\in Q)$, with support $ Q$. Observe that the distribution $P_Z$ of $Z$ is the restriction of $P_X$ to $Q$.
	Since $Q$ is compact and standard wrt $P_Z$ we can invoke Theorem 1  for $(Z_j)_{j\geq 1}$, in order to conclude that there exists a positive constant $C_Q$ depending on $Q$, such that for $k\geq k_0=k_0(\omega)$, 
	\begin{equation}
	\label{visitas_restringidas}
	d_H(\mathcal{Z}_k,Q)\leq C_Q (\log(k)/k)^{1/d},
	\end{equation}
	where $\mathcal{Z}_k=\{Z_1, \ldots, Z_k\}$. 
	Define now $V_l$ as the number of visits to the set $Q$ up to time $l$. Namely, $V_l= \sum_{i=1}^l I_{\{X_i\in Q\}}.$ 	By the law of large numbers, $V_l/l\to P(X\in Q)>0$ a.e., and therefore, for $l$ large enough, $V_l\geq k_0$. Thus, by \eqref{visitas_restringidas}, recalling that $h_l^d l/\log(l)\to \infty$, we get that 
	$$d_H(\mathcal{Z}_{V_l},Q)\leq C_Q(\log(V_l)/V_l)^{1/d}\leq \tilde C_Q(\log(l)/l)^{1/d}\leq \frac{h_l}{2}.$$ 
	In particular, $Q\subseteq \bigcup_{Z_j\in \mathcal{Z}_{V_l}} B(Z_j, h_l/2)=\bigcup_{X\in \X_{l}\cap Q} B(X,h_l/2).$
	
\end{proof}

This last lemma will be applied to get the covering properties stated in H2 and H5 for $I_a$ and $A_a^\delta$. The following results are needed to show that these sets satisfy the conditions imposed in Lemma \ref{lemcub}.

\begin{lemma} \label{lemest} Let $\nu$ be a distribution with support $I$ such that $int(I)\neq\emptyset$ and $reach(\overline{I^c})> 0$. Assume that $\nu$ has density $f$ bounded from below by $f_0> 0$. Let $Q=\overline{I\ominus B(0,\gamma)}$ such $\nu(Q)>0$, then $Q$ is standard with respect to $\nu_Q$, the restriction of $\nu$ to $Q$ (i.e $\nu_Q(A)=\nu(A\cap Q)/\nu(Q)$), for all $0\leq \gamma <reach(I^c)$, with $\beta=f_0/(3\nu(Q))$.
	
\end{lemma}

\begin{proof}  Let $0\leq \gamma < reach(\overline{I^c})$. By corollary 4.9 in \cite{federer:59} applied to $I^c$, we get that $reach(\overline{(I\ominus B(0,\gamma))^c})\geq reach(\overline{I^c})-\gamma> 0$, and now by proposition 1 in \cite{chola},  $\nu_Q$ is standard, with $\beta=f_0/(3\nu(Q))$ (see Definition \ref{st}).
\end{proof}

\begin{lemma} \label{lemconec} Let $I\subset \mathbb{R}^{d}$ be a non-empty, connected, compact set with $reach(\overline{I^c})>0$. Then for all $0< \eps\leq reach(I^c)$, $I\ominus B(0,\eps)$ is connected.
	
\end{lemma} 
\begin{proof} Let $0<\eps\leq reach(\overline{I^c})$. By corollary 4.9 in \cite{federer:59} applied to $I^c$ , $reach(I\ominus B(0,\eps))> \eps$. Then, the function $f(x)=x$ if $x\in I\ominus B(0,\eps)$, and $f(x)=\pi_{\partial (I \ominus B(0,\eps))}(x)$ if $x \in I\setminus (I\ominus B(0,\eps))$ where $\pi_{\partial S}$ denotes the metric projection onto $\partial S$, is well defined. By item 4 of theorem 4.8 in  \cite{federer:59}, $f$ is a continuous function, so it follows that $f(I)=I\ominus B(0,\eps)$ is connected.
\end{proof}

\textit{Proof of Proposition  \ref{mariela}.} \\
 Since  $reach(\overline{I_a^c})>0$ $P_X(\partial I_a)=0$ (this follows from Proposition 1 and 2 in \cite{cue12} together with Proposition 2 in \cite{chola:14}), then $\mathbb P(X\in int(I_a))=P(X\in I_a)>0$. By Lemma \ref{lemest},  choosing $\gamma=0$,  the set $\overline{I_a}$ is standard with respect to $P_X$ restricted to $\overline{I_a}$,  for $a=0,1$. By Lemma  \ref{lemcub}, with $Q=\overline{I_a}$,  $\overline{I_a}$ is coverable;  finally we get that H3 is satisfied.

 To prove H5 i) observe that the connectedness of $A_a^\delta$ follows from that of $I_a$ (H2 i) together with Lemma \ref{lemconec}. 
For H5 ii), take  $\delta$ small enough such that  $\mathbb P(X \in A_a^\delta)>0$, which should exist because of  H2 ii). 
 By (1) in  \cite{erd45}, using that $\partial A_a^\delta\subset \{x:d(x,\partial I_a)=\delta\}$, we get that  $\mathbb P(X\in \partial A_a^\delta)=0$. 
Finally to prove the covering  stated in  H5  first observe that, by Lemma \ref{lemest},    $\overline{ A_a^\delta }$ is  standard wrt $P_X$ restricted to $\overline{A_a^\delta}$. Invoking Lemma \ref{lemcub} with $Q=\overline{A_a^\delta}$ and recalling that $\mathbb P(X\in \partial A_a^\delta)=0$  we get   the covering property stated in H5 iii).  

 Lastly the uniform convergence stated in H7 follows from Theorem 6 in \cite{ab:89}, since $f$ is uniformly continuous, assumptions (i)-(iii)  hold for the uniform kernel and  the bandwidth fulfills $lh_l^{2d}/\log(l)\rightarrow \infty$.

\section*{Acknowledgments} 
We thank  two referees and an associated editor for their constructive comments and insightful suggestions, which have improved the presentation of the present manuscript. We also thank to Dami\'an Scherlis for helpful suggestions.

\end{document}